\documentclass{article}

\PassOptionsToPackage{numbers, compress}{natbib}

\usepackage[final]{neurips_2020}

\usepackage[utf8]{inputenc} 
\usepackage[T1]{fontenc}    
\usepackage{hyperref}       
\usepackage{url}            
\usepackage{booktabs}       
\usepackage{amsfonts}       
\usepackage{nicefrac}       
\usepackage{microtype}      
\usepackage{comment}
\usepackage{enumitem}

\def\name{NeuMiss }

\def\toptitlebar{\hrule height4pt \vskip .25in}
\def\bottomtitlebar{ \vskip .25in \hrule height1pt \vskip .25in}

\def\mytitle{\name networks: differentiable programming for supervised learning with missing values}
\title{\mytitle{}}
\date{}
\addtocounter{footnote}{1}

\author{Marine Le Morvan\textsuperscript{\textnormal{1,2}}\hspace{0.5em}
Julie Josse\textsuperscript{\textnormal{1,3}} \hspace{0.5em}
Thomas Moreau\textsuperscript{\textnormal{1}}\hspace{0.5em}
Erwan Scornet\textsuperscript{\textnormal{3}}\hspace{0.5em}
Gaël Varoquaux\textsuperscript{\textnormal{1, 4}}
\AND
\\[-1.5em]
\textsuperscript{1}  Université Paris-Saclay, Inria, CEA, Palaiseau, 91120, France\\
\textsuperscript{2}  Université Paris-Saclay, CNRS/IN2P3, IJCLab, 91405 Orsay, France\\
\textsuperscript{3}  CMAP, UMR7641, Ecole Polytechnique, IP Paris, 91128 Palaiseau, France\\
\textsuperscript{4}  Mila, McGill University, Montréal, Canada
\AND
\\[-1.5em]
\texttt{\{marine.le-morvan, julie.josse, thomas.moreau, gael.varoquaux\}@inria.fr}\\
\texttt{erwan.scornet@polytechnique.edu}
}
\usepackage{dsfont}
\usepackage{graphicx}
\usepackage{amssymb}
\usepackage{amsmath}
\usepackage{amstext}
\usepackage{amsthm} 
\usepackage{thm-restate}
\usepackage{bm}
\usepackage{stmaryrd} 
\usepackage[textsize=tiny]{todonotes}
\usepackage{subfiles}
\usepackage{tikz}
\usetikzlibrary{arrows, positioning, calc}
\usetikzlibrary{shapes,fit}
\usetikzlibrary{decorations.pathreplacing, patterns}

\newtheorem{assumption}{Assumption}
\newtheorem{lemma}{Lemma}

\def\toptitlebar{
\hrule height4pt
\vskip .25in}

\def\bottomtitlebar{
\vskip .25in
\hrule height1pt
\vskip .25in}

\newcommand{\RR}{\mathbb{R}}
\newcommand{\E}{\mathbb{E}}
\newcommand{\Ncal}{\mathcal{N}}
\renewcommand{\P}{\mathds{P}}
\newcommand{\NA}{\mathtt{NA}}
\newcommand{\bM}{M}
\newcommand{\bbm}{m}
\newcommand{\Xm}{\widetilde{X}}
\newcommand{\bXm}{\widetilde{X}}
\newcommand{\bx}{x}
\newcommand{\bX}{X}
\newcommand{\bxm}{\widetilde{x}}
\newcommand{\regTilde}{f^{\star}_{\bXm}}
\newcommand{\regTildel}{f^{\star}_{\bXm, \ell}}
\newcommand{\dataset}{\mathcal{D}}

\newcommand{\tmu}{\widetilde{\mu}}
\newcommand{\tsigma}{\widetilde{\sigma}}

\newcommand {\br}[1]{\left(#1\right)}
\newcommand {\sqb}[1]{\left[#1\right]}
\newcommand {\cbr}[1]{\left\{#1\right\}}
\newcommand {\bbr}[1]{\left\llbracket#1\right\rrbracket}

\newcommand{\gv}[1]
{{\todo[author=GV, color=green!20]{#1}}}

\theoremstyle{plain}\newtheorem{proposition}{Proposition}[section]

\begin{document}

\maketitle

\begin{abstract}
The presence of missing values makes supervised learning much more challenging. Indeed, previous work has shown that even when the response is a linear function of the complete data, the optimal predictor is a complex function of the observed entries and the missingness indicator. As a result, the computational or sample complexities of consistent approaches depend on the number of missing patterns, which can be exponential in the number of dimensions. In this work, we derive the analytical form of the optimal predictor under a linearity assumption and various missing data mechanisms including Missing at Random (MAR) and self-masking (Missing Not At Random). Based on a Neumann-series approximation of the optimal predictor, we propose a new \textbf{principled} architecture, named \name networks. Their originality and strength come from the use of a new type of non-linearity: the multiplication by the missingness indicator. We provide an upper bound on the Bayes risk of \name networks, and show that they have good predictive accuracy with both a number of parameters and a computational complexity independent of the number of missing data patterns. As a result they \textbf{scale well} to problems with many features, and remain \textbf{statistically efficient} for medium-sized samples. Moreover, we show that, contrary to procedures using EM or imputation, they are \textbf{robust to the missing data mechanism}, including difficult MNAR settings such as self-masking.
\end{abstract}


\section{Introduction}

Increasingly complex data-collection pipelines, often assembling
multiple sources of information, lead to datasets with incomplete
observations and complex missing-values mechanisms. The pervasiveness of missing values
has triggered an abundant statistical literature on the subject
\citep{little2002statistical,vanbuuren_2018}: a recent survey reviewed more than 150 implementations to handle missing data \citep{mayer2019rmisstastic}.
Nevertheless, most methods have been developed either for
inferential purposes, i.e. to estimate parameters of a probabilistic
model of the fully-observed data, or for imputation, completing missing entries as well as possible \citep{hastie2015softimpute}.
These methods often require strong assumptions on the missing-values
mechanism, i.e. either the missing at random (MAR) assumption \citep{rubin1976inference} --
the probability of being missing only depends on observed values -- or the more restrictive Missing Completely At Random assumption (MCAR) -- the missingness is independent of the data.
In MAR or MCAR settings, good imputation is sufficient to fit statistical
models, or even train supervised-learning models
\citep{josse2019consistency}. In particular, a precise knowledge of the
data-generating mechanism can be used to derive an Expectation
Maximization (EM) \citep{dempster1977maximum} formulation with the
minimum number of necessary parameters. Yet, as we will see, this is
intractable if the number of features is not small, as potentially $2^d$
missing-value patterns must be modeled.

The last missing-value mechanism category,
Missing Not At Random (MNAR), covers cases where the
probability of being missing depends on the unobserved values.
This is a frequent situation in which
missingness cannot be ignored in the
statistical analysis \citep{kim2018data}.
Much of the work on MNAR data focuses on problems of identifiability, in both parametric  and non-parametric settings \citep{tang2003analysis, miao2016identifiability, mohan:pea19-r473, nabi2020full}.
In MNAR settings, estimation strategies often require modeling
the missing-values mechanism \citep{ibrahim1999missing}. This
complicates the inference task and is often limited to cases with few MNAR variables.
Other approaches need the masking matrix to be well approximated
with low-rank matrices \citep{marlin2009collaborative, audibert2011robust, hernandez2014probabilistic, ma2019missing, wang2019doubly}.

Supervised learning with missing values has different goals than
probabilistic modeling \citep{josse2019consistency} and has been less
studied. As the test set is also expected to have missing entries,
optimality on the fully-observed data is no longer a goal per se. Rather, the
goal of minimizing an expected risk lend itself well to
non-parametric models which can compensate from some oddities introduced
by missing values. Indeed, with a powerful learner capable of learning
any function, imputation by a constant is Bayes consistent
\citep{josse2019consistency}. Yet, the complexity of this function that
must be approximated governs the success of this approach
outside of asymptotic regimes. In the simple case of a linear regression
with missing values, the optimal predictor has a combinatorial
expression: for $d$ features, there are $2^d$ possible missing-values
patterns requiring $2^d$ models \cite{morvan2020linear}.

\citet{morvan2020linear} showed that in this setting, a multilayer
perceptrons (MLP) can be consistent even in a pattern mixture
MNAR model, but assuming $2^d$ hidden units.
%
There have been many adaptations of neural networks to missing values,
often involving an imputation with 0's and concatenating the mask (the
indicator matrix coding for missing values)
\citep{nazabal2018handling, pmlr-v97-mattei19a, ma2018eddi,
yoon2018gain, pmlr-v118-gong20a}. However there is no theory relating the
network architecture to the impact of the missing-value mechanism on the
prediction function. In particular, an important practical question is:
how complex should the architecture be to cater for a given mechanism?
Overly-complex architectures require a lot of data, but being too
restrictive will introduce bias for missing values.

The present paper addresses the challenge of supervised learning with missing values.
We propose a theoretically-grounded neural-network architecture which allows
to implicitly impute values as a function of the observed data, aiming at
the best prediction. More precisely,
        \begin{itemize}[itemsep=0ex, topsep=0ex, partopsep=0ex, parsep=0.5ex,
		leftmargin=2.1ex]
            \item We derive an analytical expression of the Bayes predictor for linear regression in the presence of missing values under various missing data mechanisms including MAR and self-masking MNAR.

            \item We propose a new \textbf{principled} architecture, named \name network, based on a Neumann series approximation of the Bayes predictors, whose originality and strength is the use of $\odot \bM$ nonlinearities, i.e. the elementwise multiplication by the missingness indicator.

            \item We provide an upper bound on the Bayes risk of \name networks which highlights the benefits of depth and learning to approximate.

            \item We provide an interpretation of a classical ReLU network as a shallow \name network. We further demonstrate empirically the crucial role of the $\odot$ nonlinearities,
            by showing that increasing the capacity of \name networks improves predictions while it does not for classical networks.

            \item We show that \name networks are \textbf{suited
medium-sized datasets}: they require $O(d^2)$ samples,
contrary to $O(2^d)$ for methods that do not share weights between missing data patterns.

            \item We demonstrate the benefits of the proposed
architecture over classical methods such as EM algorithms or iterative
conditional imputation \citep{vanbuuren_2018} both in terms of
computational complexity --these methods scale in $O(2^d d^2)$
\citep{seber2003wiley} and  $O( d^3)$ respectively--, and in the ability to be \textbf{robust to the missing data mechanism}, including MNAR.


        \end{itemize}





\section{Optimal predictors in the presence of missing values}

\paragraph{Notations}

We consider a data set $\dataset_n = \{(\bX_1, Y_1), \hdots, (\bX_n, Y_n)\}$ of independent pairs $(\bX_i, Y_i)$, distributed as the generic pair $(\bX, Y)$, where $\bX \in \RR^d$ and $Y \in \RR$. We introduce the indicator
vector $\bM \in \{0,1\}^d$ which satisfies, for all $1 \leq j \leq d$,
$M_j =1$ if and only if $X_j$ is not observed. The random vector $\bM$
acts as a mask on $\bX$. 
We define the incomplete feature vector $\bXm \in \widetilde{\mathcal X} =  (\RR \cup \{\NA\})^d$ (see
\cite{rubin1976inference}, \cite[appendix B]{rosenbaum_rubin_JASA1984}) 
as $\Xm_{j}=\NA$ if $M_{j}=1$, and  $\Xm_{j}=X_j$ otherwise.
As such, $\bXm$ is a mixed categorical and continuous variable. An
example of realization (lower-case letters) of the previous random
variables would be a vector $\bx = (1.1, 2.3, -3.1, 8, 5.27)$ with the
missing pattern $m = (0,1,0,0,1)$, giving
$
\bxm = (1.1, ~~\texttt{NA}, ~-3.1, ~~8, ~~\texttt{NA}).
$

For realizations $\bbm$ of $\bM$, we also denote by
$obs(\bbm)$ (resp. $mis(\bbm)$) the indices of the zero entries of $m$
(resp. non-zero). Following classic missing-value notations, we let
$\bX_{obs(\bM)}$ (resp. $\bX_{mis(\bM)}$)
be the observed (resp. missing) entries in $\bX$. Pursuing
the above example, we have $mis(\bbm)=\{1, 4\}$, $obs(\bbm)=\{0, 2, 3\}$,
$\bx_{obs(\bbm)} = (1.1, -3.1, ~~8)$, $\bx_{mis(\bbm)} = (  2.3,
~~5.27)$.
To lighten notations, when there is no ambiguity, we remove the explicit dependence in $\bbm$ and write, e.g.,
$\bX_{obs}$.

\subsection{Problem statement: supervised learning with missing values}

We consider a linear model of the complete data, such that the response $Y$ satisfies:
\begin{align}
    Y &= \beta_0^{\star} + \langle \bX, \beta^{\star} \rangle + \varepsilon,
    \qquad
\text{for some}\; \beta_0^\star \in \mathbb{R}, \beta^\star \in
\mathbb{R}^d, \;\text{and}\; \varepsilon\sim\mathcal N(0, \sigma^2).
\label{eq:predictor}
\end{align}
Prediction with missing values departs from standard linear-model
settings: the aim is
to predict $Y$ given $\bXm$, as the complete input $\bX$ may be unavailable. The corresponding optimization problem is:
\begin{align}
    \regTilde \in \underset{f:\widetilde{\mathcal X} \rightarrow \mathbb R}{\mathrm{argmin}}~
    \E[(Y - f(\bXm))^2], \label{eq_optimisation_pb}
\end{align}
where $\regTilde$ is the Bayes predictor for the squared loss, in the presence of missing values. The main difficulty of this problem comes from the half-discrete nature of the input space $\widetilde{\mathcal X}$. Indeed, the Bayes predictor $\regTilde(\bXm) = \E \big[Y~|~\bXm \big]$ can be rewritten as:
\begin{equation}
 \label{eq:bp_submodels}
    \regTilde(\bXm) = \E \left[Y~|~\bM, \bX_{obs(\bM)}\right] =  \sum_{\bbm\in\{0,1\}^d}
\E\left[Y | \bX_{obs(\bbm)}, \bM=\bbm\right] ~ \mathds{1}_{\bM=\bbm},
\end{equation}
which highlights the combinatorial issue of solving
\eqref{eq_optimisation_pb}: one may need to
optimize $2^d$ submodels, for the different $m$.
In the
following,
we write the Bayes predictor $f^{\star}$ as a function of
$(X_{obs(M)},M)$:
\begin{align*}
f^{\star}(X_{obs(M)},M) = \E \left[Y | \bX_{obs(M)}, \bM\right].  \label{eq:Bayes_predictor_xobs}
\end{align*}

\subsection{Expression of the Bayes predictor under various
missing-values mechanisms}
\label{sec:theory_Bayes}




There is no general closed-form expression for the Bayes predictor, as it depends on the data distribution and missingness mechanism. However, an exact expression can be derived for Gaussian data with various missingness mechanisms.

\begin{assumption}[Gaussian data]\label{ass:gaussian}
The distribution of $\bX$ is Gaussian, that is, $\bX \sim \mathcal{N} (\mathbf{\mu}, \Sigma)$.
\end{assumption}

\begin{assumption}[MCAR mechanism]
\label{ass:mcar}
For all $\bbm \in \{0,1\}^d$, $P(\bM = \bbm|\bX) = P(\bM = m)$.
\end{assumption}

\begin{assumption}[MAR mechanism]
\label{ass:mar}
For all $\bbm \in \{0,1\}^d$, $P(\bM = \bbm|\bX) = P(\bM = \bbm|\bX_{obs(\bbm)})$.
\end{assumption}

\begin{restatable}[MAR Bayes predictor]{proposition}{MARbp}
\label{prop:MCAR_MAR}
Assume that the data are generated via the linear model defined in equation~\eqref{eq:predictor} and satisfy Assumption~\ref{ass:gaussian}. Additionally, assume that either Assumption~\ref{ass:mcar} or Assumption~\ref{ass:mar} holds. Then the Bayes predictor $f^{\star}$ takes the form
\begin{equation}
    f^{\star}(X_{obs}, M) =  \beta_0^{\star} + \langle \beta_{obs}^{\star},  \bX_{obs} \rangle + \langle \beta_{mis}^{\star}, \mu_{mis} + \Sigma_{mis, obs} (\Sigma_{obs})^{-1} (X_{obs}-\mu_{obs}) \rangle,
    \label{eq:bp_MCAR}
\end{equation}
where we use $obs$ (resp. $mis$) instead of $obs(M)$ (resp. $mis(M)$) for lighter notations.
\end{restatable}

Obtaining the Bayes predictor expression turns out to be far more complicated for general MNAR settings but feasible for the Gaussian self-masking mechanism described below.

\begin{assumption}[Gaussian self-masking]\label{ass:selfmasked_gaussian}
The missing data mechanism is self-masked with $P(M|X) = \prod_{k=1}^d P(M_k|X_k)$ and $\forall k \in \bbr{1, d},$
$$ P(M_k = 1|X_k) = K_k \exp\br{-\frac{1}{2} \frac{(X_k -
\tmu_k)^2}{\tsigma_k^2}}\qquad
\text{with}\; 0<K_k<1.
$$
\end{assumption}



\begin{restatable}[Bayes predictor with Gaussian self-masking]{proposition}{GaussianSMbp}
\label{prop:sm}
Assume that the data are generated via the linear model defined in equation~\eqref{eq:predictor} and satisfy Assumption~\ref{ass:gaussian} and Assumption~\ref{ass:selfmasked_gaussian}.  Let $
\Sigma_{mis|obs} = \Sigma_{mis, mis} - \Sigma_{mis, obs} \Sigma_{obs}^{-1} \Sigma_{obs, mis},$ and let $D$ be the diagonal matrix such that $\mathrm{diag}(D) = (\tsigma_1^2, \hdots, \tsigma_d^2)$. Then the Bayes predictor writes
\begin{align}
 \label{eq:bp_MNAR}
    f^{\star}(X_{obs},M)  & =  \beta_0^{\star} + \langle \beta_{obs}^{\star},  \bX_{obs} \rangle + \langle \beta_{mis}^{\star},  (Id + D_{mis}\Sigma_{mis| obs}^{-1})^{-1} \nonumber \\
      & \quad \times
   ( \tilde{\mu}_{mis} + D_{mis}\Sigma_{mis| obs}^{-1} (\mu_{mis} + \Sigma_{mis, obs} \br{\Sigma_{obs}}^{-1}\br{X_{obs}-\mu_{obs}}))
\rangle
\end{align}
\end{restatable}

The proof of Propositions~\ref{prop:MCAR_MAR} and \ref{prop:sm} are
in the Supplementary Materials (\ref{app:sec:proof_prop1} and \ref{app:proof_prop_mnar}). These are the first results establishing
exact expressions of the Bayes predictor in a MAR and specific MNAR
mechanisms.
Note that these propositions show that the Bayes predictor is linear by
pattern under the assumptions studied, i.e., each of the $2^d$ submodels
in equation \ref{eq:bp_submodels} are linear functions of $X_{obs}$. For
non-Gaussian data, the Bayes predictor may not be
linear by pattern \cite[Example~3.1]{morvan2020linear}.

\paragraph{Generality of the Gaussian self-masking model}
For a self-masking mechanism where
the probability of being missing increases (or decreases) with the value
of the underlying variable, probit or logistic functions are often
used \citep{kim2018data}. A Gaussian self-masking model is also
a suitable model: setting the mean of the Gaussian close to the
extreme values gives a similar behaviour. In addition, it covers cases where the
probability of being missing is centered around a given value.





\section{\name networks: learning by approximating the Bayes predictors}


\subsection{Insight to build a network: sharing parameters across
missing-value patterns}

%
%
Computing the Bayes predictors in
equations~\eqref{eq:bp_MCAR} or \eqref{eq:bp_MNAR} requires to estimate
the inverse of each submatrix $\Sigma_{obs(m)}$ for each missing-data pattern
$m \in \{0,1\}^d$, \emph{ie} one linear model per missing-data
pattern.
For a number of hidden units $\propto 2^d$, a
MLP with ReLU non-linearities can fit these linear models independently
from one-another, and is shown to be consistent
\citep{morvan2020linear}. But it is prohibitive when $d$ grows.
Such an architecture is largely over-parametrized as it does not
share information between similar missing-data patterns.
Indeed, the slopes of each of the linear regression per pattern given by the Bayes
predictor in equations \eqref{eq:bp_MCAR} and \eqref{eq:bp_MNAR} are
linked via the inverses of $\Sigma_{obs}$.

Thus, one approach is to
estimate only one vector $\mu$ and one covariance matrix $\Sigma$ via an
expectation maximization (EM) algorithm \citep{dempster1977maximum}, and
then compute the inverses of $\Sigma_{obs}$. But the computational
complexity then scales linearly in the
number of missing-data patterns (which is in the worst case exponential in the
dimension $d$), and is therefore also prohibitive when the dimension increases.

In what follows, we propose an in-between
solution, modeling the relationships between the slopes for
different missing-data patterns without directly estimating
the covariance matrix. 
Intuitively, observations from one pattern will be used to estimate the regression parameters of other patterns.

\subsection{Differentiable approximations of the inverse covariances with Neumann series}

The major challenge of equations~\eqref{eq:bp_MCAR} and
\eqref{eq:bp_MNAR} is the inversion
of the matrices $\Sigma_{obs(m)}$ for all $m \in \{0,1\}^d$. Indeed,
there is no simple relationship for the inverses of different submatrices
in general. As a result, the slope corresponding to a pattern $m$ cannot be
easily expressed as a function of $\Sigma$. 

We therefore propose to approximate $\br{\Sigma_{obs(m)}}^{-1}$ for all $m \in
\{0,1\}^d$ recursively in the following way. First, we choose as a
starting point a $ d\times d$ matrix $S^{(0)}$. $S^{(0)}_{obs(m)}$ is
then defined as the sub-matrix of $S^{(0)}$ obtained by selecting the
columns and rows that are observed (components for which $m=0$) and is
our order-$0$ approximation of $\br{\Sigma_{obs(m)}}^{-1}$. Then, for all
$m \in \{0,1\}^d$, we define the order-$\ell$ approximation $S^{(\ell)}_{obs(m)}$ of $\br{\Sigma_{obs(m)}}^{-1}$ via the following iterative formula: for all $\ell \geq 1$,
\begin{equation}
S^{(\ell)}_{obs(m)} = (Id - \Sigma_{obs(m)})\,S^{(\ell-1)}_{obs(m)} + Id.
    \label{eq:iterative_algo}
\end{equation}
The iterates $S^{(\ell)}_{obs(m)}$ converge linearly to $(\Sigma_{obs(m)})^{-1}$(\ref{sec:conv_neumann_iterates} in the Supplementary
Materials), and are in fact Neumann series truncated to $\ell$ terms if $S^{(0)} = Id$. 

 We now define the
order-$\ell$ approximation of the Bayes predictor in MAR settings (equation
~\eqref{eq:bp_MCAR}) as
\begin{equation}
     f^{\star}_{\ell}(X_{obs}, M) =  \langle \beta^\star_{obs},  \bX_{obs} \rangle + \langle \beta^\star_{mis}, \mu_{mis} + \Sigma_{mis,obs}S^{(\ell)}_{obs(m)} (X_{obs} - \mu_{obs}) \rangle.
    \label{eq:bp_MCAR_any_order}
\end{equation}
The error between the Bayes predictor and its order-$\ell$ approximation is provided in Proposition~\ref{prop:bayes_bound_approx}.
\begin{restatable}{proposition}{approxBayes}
    \label{prop:bayes_bound_approx}
    Let $\nu$ be the smallest eigenvalue of $\Sigma$. Assume that the data are generated via a linear model defined in equation~\eqref{eq:predictor} and satisfy Assumption~\ref{ass:gaussian}. Additionally, assume that either Assumption~\ref{ass:mcar} or Assumption~\ref{ass:mar} holds and that the spectral radius of $\Sigma$ is strictly smaller than one. Then, for all $\ell \geq 1$,
    \begin{align}
        \E\biggl[\bigl(  f_{\ell}^{\star}(X_{obs}, M) -
f^{\star}(X_{obs}, M) \bigr)^2\biggr] \;\le\; \frac{(1- \nu)^{2 \ell} \|\beta^{\star}\|_2^2}{\nu} \,
                    \E\biggl[ \bigl\|Id - S^{(0)}_{obs(M)}
\Sigma_{obs(M)}\bigr\|_2^2\biggr]
                    \label{eq:prop_exponential_decay_approx}
    \end{align}
\end{restatable}
The error of the order-$\ell$ approximation decays exponentially fast with $\ell$.
More importantly, if the submatrices $S^{(0)}_{obs}$ of $S^{(0)}$ are good approximations of  $(\Sigma_{obs})^{-1}$ on average, that is if we choose $S^{(0)}$ which minimizes the expectation in the right-hand side in inequality~\eqref{eq:prop_exponential_decay_approx}, then our model provides a good approximation of the Bayes predictor even with order $\ell = 0$.
This is the case for a diagonal covariance matrix, as taking $S^{(0)}=
\Sigma^{-1}$ has no approximation error as $(\Sigma^{-1})_{obs} = (\Sigma_{obs})^{-1}$.


\subsection{\name network architecture: multiplying by the mask}
\paragraph{Network architecture} We propose a neural-network architecture to approximate the Bayes predictor, where the inverses $(\Sigma_{obs})^{-1}$ are computed using an unrolled version of the iterative algorithm. Figure~\ref{fig:unrolled_NN} gives  a diagram for such neural network using an order-3
approximation corresponding to a depth 4. $x$ is the input, with missing
values replaced by 0. $\mu$ is a trainable parameter corresponding to the parameter $\mu$ in equation~\eqref{eq:bp_MCAR_any_order}. To match the Bayes predictor exactly (equation~\eqref{eq:bp_MCAR_any_order}), weight matrices should be simple transformations of the covariance matrix indicated in blue on Figure~\ref{fig:unrolled_NN}.

Following strictly Neummann iterates would call for
a shared weight matrix across
all $W_{Neu}^{(k)}$. Rather, we learn each layer independently. This choice is motivated by works on iterative algorithm unrolling \cite{gregor2010learning} where  independent layers' weights can improve a network's approximation performance \cite{Xin2016}. Note that \cite{Gilton2020Neumann} has also introduced a neural network architecture based on unrolling the Neumann series. However, their goal is to solve a linear inverse problem with a learned regularization, which is very different from ours.
%

\paragraph{Multiplying by the mask} Note that the observed indices change for each sample, leading to an implementation challenge. For a sample with missing data pattern $m$,  the weight matrices $S^{(0)}$, $W_{Neu}^{(1)}$ and
$W_{Neu}^{(2)}$ of Figure~\ref{fig:unrolled_NN} should be masked such
that their rows and columns corresponding to the indices $mis(m)$ are
zeroed, and the rows of $W_{Mix}$ corresponding to $obs(m)$ as well as
the columns of $W_{Mix}$ corresponding to $mis(m)$ are zeroed.
Implementing efficiently a network in which the weight matrices are masked
differently for each sample can be challenging. We thus use the following trick.
Let $W$ be a weight matrix, $v$ a vector, and $\bar{m} = 1 -
m$. Then $(W \odot \bar{m}\bar{m}^\top) v = (W (v\odot \bar{m}))\odot
\bar{m}$, i.e, using a masked weight matrix is equivalent to masking the
input and output vector.
The network can then be seen as a classical network where the nonlinearities
are multiplications by the mask.

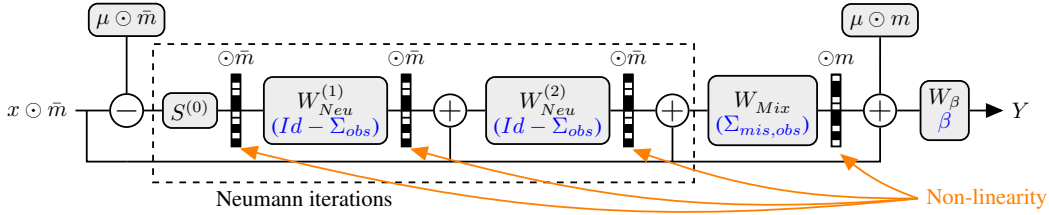
\begin{figure}[b]
    \centering
    \scalebox{.85}{\begin{tikzpicture}[thick, >=triangle 45, scale=0.8]

    \def\learncolor{black!7}
    \tikzset{%
        learnable/.style = {
            fill=\learncolor,
        },
        parameter/.style = {
            draw, rectangle,
            align=center,
           rounded corners,
           learnable
        },
        layer/.style = {
            parameter,
        },
        mu/.style = {
            draw, rectangle,
            align=center,
           fill=white,
           rounded corners,
           learnable
        },
        sum/.style = {
            draw, circle,
            inner sep=0pt
        },
        pil/.style={
               ->,
               thick,
        }
    }

    \def\interlayer{.8em}
    \def\intermask{2em}

    \node(X_obs){$x \odot \bar m$};
    \node[right=1.5em of X_obs, sum] (add1) {\Large$-$};
    \node[above=2.5em of add1, mu] (mu_obs){$\mu \odot \bar m$};
    \node[right=\interlayer of add1, parameter] (S0) {$S^{(0)}$};
    \node[right=\intermask of S0, layer] (layer1) {$W_{Neu}^{(1)}$\\ \textcolor{blue}{($Id - \Sigma_{obs}$)}};
    \node[right=\intermask of layer1, sum] (add2) {\Large$+$};
    \node[right=\interlayer of add2, layer] (layer2) {$W_{Neu}^{(2)}$\\ \textcolor{blue}{($Id - \Sigma_{obs}$)}};
    \node[right=\intermask of layer2, sum] (add3) {\Large$+$};
    \node[right=\interlayer of add3, layer] (layer3) {$W_{Mix}^{\textcolor{\learncolor}{(3)}}$\\ \textcolor{blue}{($\Sigma_{mis,obs}$)}};
    \node[right=\intermask of layer3, sum] (add4) {\Large$+$};
    \node[above=2.5em of add4, mu] (mu_obs2){$\mu \odot m$};
    \node[right=1em of add4, parameter] (layer4) {$W_\beta$\\ \textcolor{blue}{$\beta$}};
    \node[right=1.5em of layer4] (Y) {$Y$};
    
    \draw[pil] (X_obs) -- (add1) -- (S0) coordinate[midway] (rect1) 
                  -- (layer1) coordinate[pos=.3] (mask0)
                  -- (add2) coordinate[pos=.3] (mask1) -- (layer2)
                  -- (add3) coordinate[pos=.3] (mask2) -- (layer3)
                  -- (add4) coordinate[pos=.3] (mask3) -- (layer4)
                  -- (Y);
    \draw (mu_obs) -- (add1);
    \draw (mu_obs2) -- (add4);

    \coordinate[left=1em of add1] (mock){};
    \coordinate[below=1.5em of add1] (start){};
    \draw (mock) |- (start);
    \foreach \i in {2, 3, 4}{
        \draw (start) -| (add\i) coordinate[midway] (start) {};
    
    }
    
    \def\height{2em}
    \def\width{.5em}
    \def\nrect{10}
    \def\mask{0, 1, 2, 3}
    \foreach \k in \mask{
    
        \ifthenelse{\k = 3}{
            \def\blackidx{1, 3, 4, 7, 8}
            \def\masklabel{$\odot m$}
        }{
            \def\blackidx{0, 2, 5, 6, 9}
            \def\masklabel{$\odot \bar m$}
        }
        \draw (mask\k.west)
            -- ++(0, \height)
            -- ++(\width, 0) node[midway, above] {\masklabel}
            -- ++(0, -2*\height)
            -- ++(-\width, 0) node (anchor\k) {}
            -- cycle;
    
        \foreach \i in {1,...,\nrect}{
            \draw ($ (anchor\k.center) + (0, 2*\i*\height/\nrect)$)
                -- ++(\width, 0);
        }
        
        \foreach \i in \blackidx{
            \fill[black] ($ (anchor\k.center) + (0, 2*\i*\height /\nrect)$)
                -- ++(0, 2*\height/\nrect)
                -- ++(\width, 0)
                -- ++(0, -2*\height/\nrect)
                -- cycle;
        
        }
    }
    
    \coordinate[right=.2em of add3] (rect2){};
    \draw[dashed]
        (rect1) |-($(rect1) +(3em, 3.7em)$) -- ($(rect2) + (-3em, 3.7em)$)
        -| (rect2) |-($(rect2) -(4em, 4em)$) -- ($(rect1) - (-4em, 4em)$)
            node[below, pos=.8] (desc) {Neumann iterations}
        -| cycle;
        
    \node[right=23em of desc, orange] (labelNL) {Non-linearity};
    
    \foreach \k in \mask{
        \draw[orange,->] (labelNL.west) to[bend left=10] (anchor\k);
    }

\end{tikzpicture}}
    \caption{\textbf{\name network architecture with a depth of 4} --- $\bar{m} = 1-m$. Each weight matrix $W^{(k)}$ corresponds to a simple transformation of the covariance matrix indicated in blue.}
    \label{fig:unrolled_NN}
\end{figure}

\paragraph{Approximation of the Gaussian self-masking Bayes predictor} Although our architecture is motivated by the expression of the Bayes predictor in MCAR and MAR settings, a similar architecture can be used to target the prediction function \eqref{eq:bp_MNAR} for self-masking data.
To see why, let's first assume that $D_{mis}\Sigma_{mis| obs}^{-1} \approx Id$. Then, the self-masking Bayes predictor \eqref{eq:bp_MNAR} becomes:
\begin{align}
 \label{eq:bp_MNAR_approx1}
    f^{\star}(X_{obs},M) &\approx \beta_0^{\star} + \bigl\langle \beta_{obs}^{\star},  \bX_{obs} \rangle \nonumber\\
    & \quad + \langle \beta_{mis}^{\star},
   \frac{1}{2} (\tilde{\mu}_{mis} + \mu_{mis}) + \frac{1}{2} \Sigma_{mis, obs} \br{\Sigma_{obs}}^{-1}\br{X_{obs}-\mu_{obs}}\bigr\rangle
\end{align}
i.e., its expression is the same as for the M(C)AR Bayes predictor \eqref{eq:bp_MCAR} except that $\mu_{mis}$ is replaced by $\frac{1}{2} (\tilde{\mu}_{mis} + \mu_{mis})$ and $\Sigma_{mis, obs}$ is scaled down by a factor $\frac{1}{2}$. Thus, under this approximation, the self-masking Bayes predictor can be modeled by our proposed architecture (just as the M(C)AR Bayes predictor), the only difference being the targeted values for the parameters $\mu$ and $W_{mix}$ of the network. A less coarse approximation also works: $D_{mis}\Sigma_{mis| obs}^{-1} \approx \hat{D}_{mis}$ where $\hat{D}$ is a diagonal matrix. In this case, the proposed architecture can perfectly model the self-masking Bayes predictor: the parameter $\mu$ of the network should target $(Id + \hat{D})^{-1} (\tilde{\mu} + \hat{D}\mu)$ and $W_{mix}$ should target $(Id +\hat{D})^{-1}\hat{D}\,\Sigma$ instead of simply $\Sigma$ in the M(C)AR case. Consequently, our architecture can well approximate the self-masking Bayes predictor by adjusting the values learned for the parameters $\mu$ and $W_{mix}$ if $D_{mis}\Sigma_{mis| obs}^{-1}$ are close to diagonal matrices.



\subsection{Link with the multilayer perceptron with ReLU activations}

A common practice to handle missing values
is to consider as input the data concatenated with the mask
\textsl{eg} in  \cite{morvan2020linear}. The next proposition connects
this practice to Neumman networks.

\begin{restatable}[equivalence MLP - depth-1 \name network]{proposition}{eqNeumannMLP}
\label{prop:link_MLP_Neumann}
Let $\sqb{X\odot(1-M), M} \in [0,1]^d \times \cbr{0, 1}^{d}$ be an input $X$ imputed by 0 concatenated with the mask $M$.
\begin{itemize}[itemsep=0ex, topsep=0ex, partopsep=0.5ex, parsep=0.5ex]
    \item Let $\mathcal{H}_{ReLU} = \br{W \in \RR^{d \times 2d}, ReLU}$ be a hidden layer which connects $\sqb{X\odot(1-M), M}$ to $d$ hidden units, and applies a ReLU nonlinearity to the activations.

    \item Let $\mathcal{H}_{\odot M} = \br{W \in \RR^{d \times d}, \mu, \odot M}$ be a hidden layer that connects an input $(X-\mu) \odot (1-M)$ to $d$ hidden units, and applies a $\odot M$ nonlinearity.
\end{itemize}
Denote by $h^{ReLU}_k$ and  $h^{\odot M}_k$ the outputs of the $k^{th}$ hidden unit of each layer. Then there exists a configuration of the weights of the hidden layer $\mathcal{H}_{ReLU}$ such that $\mathcal{H}_{\odot M}$ and $\mathcal{H}_{ReLU}$ have the same hidden units activated for any $(X_{obs}, M)$, and activated hidden units are such that $h^{ReLU}_k(X_{obs}, M) = h^{\odot M}_k(X_{obs}, M) + c_k$ where $c_k \in \RR$.
\end{restatable}

Proposition~\ref{prop:link_MLP_Neumann} states that a hidden layer
$\mathcal{H}_{ReLU}$ can be rewritten as a $\mathcal{H}_{\odot M}$ layer
up to a constant. Note that, as soon as another layer is stacked after
$\mathcal{H}_{\odot M}$ or  $\mathcal{H}_{ReLU}$, this additional
constant can be absorbed into the biases of this new layer.
Thus the weights of $\mathcal{H}_{ReLU}$ can be
learned so as to mimic $\mathcal{H}_{\odot M}$. In our case, this means
that a MLP with ReLU activations, one hidden layer of $d$ hidden units,
and which operates on the concatenated vector, is closely related to
the $1$-depth \name network (see Figure~\ref{fig:unrolled_NN}),
thereby providing theoretical support for the use of the latter MLP. This
theoretical link completes the results of
\cite{morvan2020linear}, who showed experimentally that in such a MLP
$O(d)$ units were enough to perform well on Gaussian data, but only provided theoretical results with $2^d$ hidden units.

\section{Empirical results}

\subsection{The $\odot M$ nonlinearity is crucial to the performance}
The specificity of \name networks resides in the $\odot M$
nonlinearities, instead of more conventional choices such as ReLU.
Figure~\ref{fig:depth_effect} shows how the choice of nonlinearity
impacts the performance as a function of the depth. We compare two
networks that take as input the data imputed by 0 concatenated with
the mask:  MLP Deep which has  1 to 10 hidden layers of $d$ hidden units followed by ReLU nonlinearities  and MLP Wide which has  one hidden layer whose width is increased followed by a ReLU nonlinearity. This latter  was shown to be consistent given $2^d$ hidden units \citep{morvan2020linear}.

Figure~\ref{fig:depth_effect} shows that increasing the capacity (depth) of MLP Deep
fails to improve the performances, unlike with
\name networks. Similarly, it is also significantly more effective to increase the capacity of the \name network (depth) than to increase the capacity (width) of MLP Wide.
These results highlight the crucial role played by
the $\odot$ nonlinearity.
Finally, the performance of MLP Wide with $d$ hidden units is close to
that of \name with a depth of 1, suggesting that it may rely on the weight
configuration established in Proposition~\ref{prop:link_MLP_Neumann}.

\begin{figure}

\begin{minipage}{.35\linewidth}
    \caption{\textbf{Performance as a function of capacity across
architectures} --- Empirical evolution of the performance for a linear
generating mechanism in MCAR settings.
Data are generated under a linear model with Gaussian covariates in a
MCAR setting (50\% missing values, $n=10^5$, $d=20$).
\label{fig:depth_effect}}
\end{minipage}%
\hfill%
\begin{minipage}{.58\linewidth}
    \includegraphics[width=\linewidth]{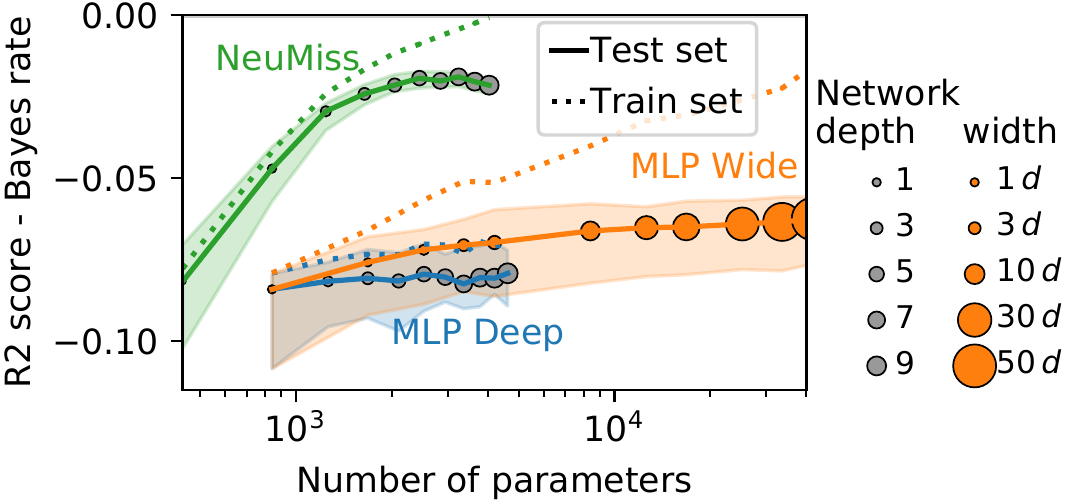}
\end{minipage}
\end{figure}

\subsection{Approximation learned by the \name network}
The \name architecture was designed to approximate well the Bayes predictor
(\ref{eq:bp_MCAR}). As shown in
Figure~\ref{fig:unrolled_NN}, its weights can be chosen so as to express
the Neumann approximation of the Bayes predictor
(\ref{eq:bp_MCAR_any_order}) exactly. We will call this particular
instance of the network, with $S^{(0)}$ set to identity, the analytic network. However, just like LISTA \citep{gregor2010learning}
learns improved weights compared to the ISTA iterations, the \name
network may learn improved weights compared to the Neumann iterations.
Comparing the performance of the analytic network to its learned
counterpart on simulated MCAR data, Figure \ref{fig:different_neumann} (left)
shows that the learned network requires a much smaller depth compared to
the analytic network to reach a given performance. Moreover, the depth-1
learned network largely outperforms the depth-1 analytic network, which
means that it is able to learn a good initialization $S^{(0)}$
for the iterates. Figure \ref{fig:different_neumann} also compares the
performance of the learned network with and without residual connections,
and shows that residual connections are not needed for good performance. This observation is another hint that the iterates learned by the network depart from the Neumann ones.

\subsection{\name networks require $O(d^2)$ samples}

Figure \ref{fig:different_neumann} (right) studies the depth for which \name
networks perform well for different number of samples $n$ and features
$d$. It outlines that \name networks work well in regimes with more
than 10 samples available per model parameters, where the number of model
parameters scales as $d^2$. In general, even with many samples, depth of
more than 5 explore diminishing returns.
Supplementary figure \ref{fig:different_neumann_mnar} shows the same
behavior in various MNAR settings.

\begin{figure}[ht]
    \includegraphics[trim={0.5cm 0 0 0.25cm}, height=.317\linewidth]{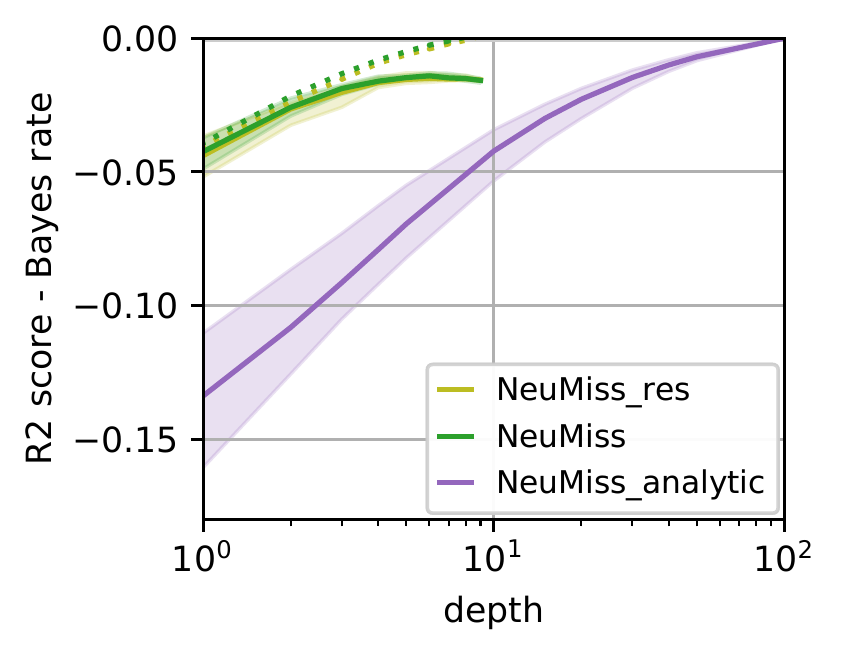}%
    \quad
    \includegraphics[height=.317\linewidth]{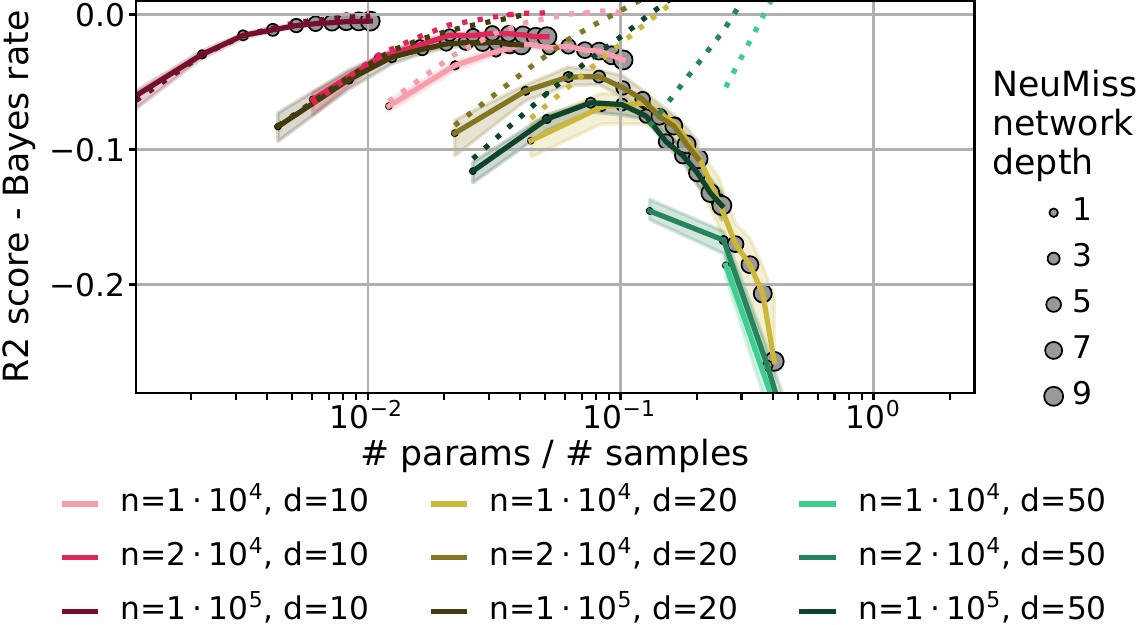}%
    \llap{\raisebox{.29\linewidth}{\sffamily MCAR}\,\qquad\qquad}%

    \caption{\textbf{Left: learned versus analytic Neumann iterates} --- \name
analytic is the \name architecture with weights set to represent \eqref{eq:iterative_algo}, supposing we have access to the ground truth parameters, \name (resp. \name res) corresponds to the network without (resp. with) residual connections.
\qquad
\textbf{Right: Required capacity in various settings} ---
Performance of \name networks varying the depth in simulations with
different number of samples $n$ and of features $d$.
\label{fig:different_neumann}}
\end{figure}

\subsection{Prediction performance: \name networks are robust to the missing data mechanism}

We now evaluate the performance of \name networks compared to other
methods under various missing values mechanisms. The data are generated
according to a multivariate Gaussian distribution, with a covariance matrix
$\Sigma = UU^\top + \text{diag}(\epsilon)$, $U \in \RR^{d \times
\frac{d}{2}}$, and the entries of $U$ drawn from a standard normal
distribution. The noise $\epsilon$ is a vector of entries drawn uniformly in $\sqb{10
^{-2}, 10^{-1}}$ to make $\Sigma$ full rank. The mean is drawn from a
standard normal distribution. The response $Y$ is generated as a linear
function of the complete data $X$ as in equation~\ref{eq:predictor}. The
noise is chosen to obtain a signal-to-noise ratio of 10.  50\% of entries
on each features are missing, with various missing data mechanisms: MCAR, MAR, Gaussian self-masking and Probit self-masking. The Gaussian self-masking is obtained according to Assumption ~\ref{ass:selfmasked_gaussian}, while the Probit self-masking is a similar setting where the probability for feature $j$ to be missing depends on its value $X_j$ through an inverse probit function.
We compare the performances of the following methods:
\begin{itemize}[itemsep=0ex, topsep=0ex, partopsep=0.5ex, parsep=0.5ex,
		leftmargin=2.5ex]
     \item \textbf{EM}: an Expectation-Maximisation algorithm
\citep{norm} is run to estimate the parameters of the joint probability
distribution of $X$ and $Y$ --Gaussian-- with missing values. Then based on this estimated distribution, the prediction is given by taking the expectation of $Y$ given $X$.

    \item \textbf{MICE + LR}: the data is first imputed using conditional
imputation as implemented in scikit-learn's \citep{scikit-learn} IterativeImputer, which proceeds by iterative ridge regression. It adapts the well known MICE \citep{vanbuuren_2018} algorithm to be able to impute a test set. A linear regression is then fit on the imputed data.

    \item \textbf{MLP}: A multilayer perceptron as in
\cite{morvan2020linear}, with one hidden layer followed by a ReLU
nonlinearity, taking as input the data imputed by 0 concatenated with the
mask. The width of the hidden layer is varied between $d$ and $100\,d$
hidden units, and chosen using a validation set. The MLP is trained using
ADAM and a batch size of 200. The learning rate is initialized to $\frac{10^{-2}}{d}$ and decreased by a factor of 0.2 when the loss stops decreasing for 2 epochs. The training finishes when either the learning rate goes below $5\times 10^{-6}$ or the maximum number of epochs is reached.

    \item \textbf{\name}: The \name architecture, without residual
connections, choosing the depth on a validation set.
The architecture was implemented using PyTorch
\cite{paszke2019pytorch}, and
optimized using stochastic gradient descent and a batch size of 10. The learning rate schedule and stopping criterion are the same as for the MLP.
\end{itemize}

\begin{figure}
    {\sffamily\small\hspace*{.08\linewidth} MCAR
    \hfill Gaussian self-masking
    \hspace*{.07\linewidth} Probit self-masking\hspace*{.1\linewidth}}\vspace*{-.1ex}

    \includegraphics[height=.34\linewidth]{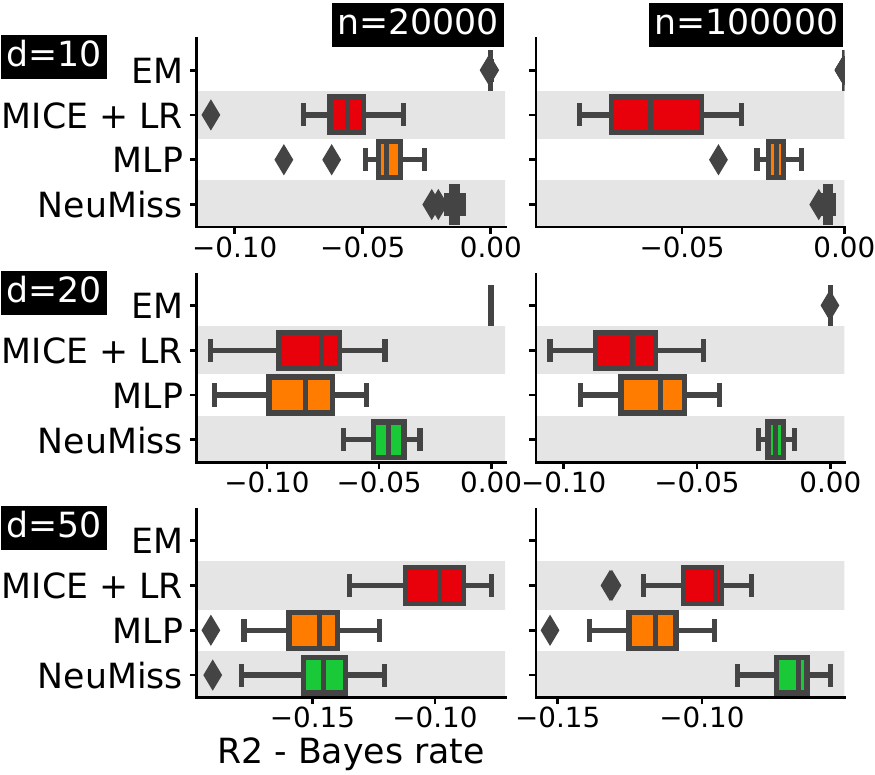}%
    \hfill%
    \includegraphics[clip, trim={1.9cm 0 0 0},
	height=.34\linewidth]{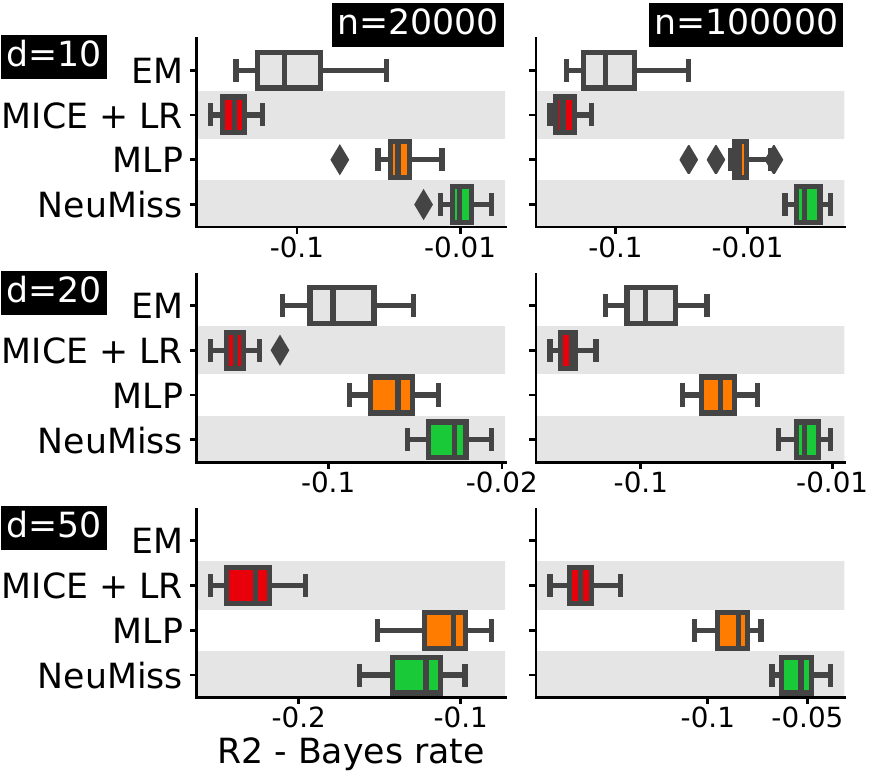}%
    \hfill%
    \includegraphics[clip, trim={1.9cm 0 0 0},
	height=.34\linewidth]{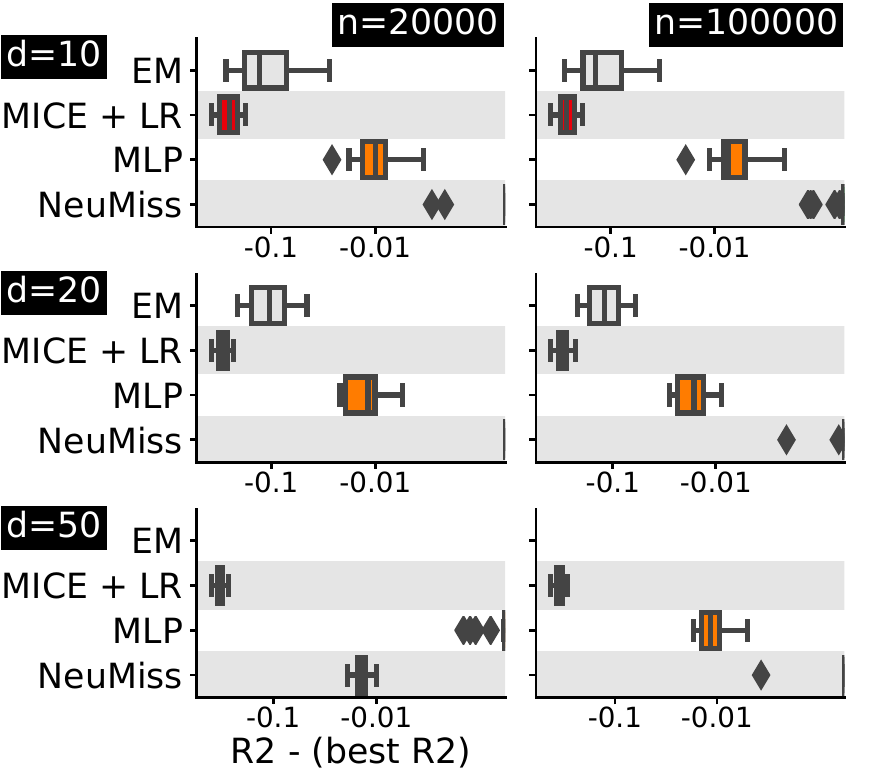}%
\caption{\textbf{Predictive performances in various scenarios} ---
varying missing-value mechanisms, number of samples $n$, and number
of features $d$. All experiments are repeated 20 times. For self-masking
settings, the x-xaxis is in log scale, to accommodate
the large difference between methods.
\label{fig:boxplotmechanism}
}
\end{figure}



For MCAR, MAR, and Gaussian self-masking settings, the performance is given
as the obtained R2 score minus the Bayes rate (the closer to 0 the
better), the best achievable R2 knowing the
underlying ground truth parameters. In our experiments, an estimation of the Bayes rate is obtained using the score of the Bayes predictor. For probit self-masking, as we  lack an analytical expression for the Bayes predictor, the performance is given with respect to the best performance achieved across all methods. The code to reproduce the experiments is available in GitHub \footnote{https://github.com/marineLM/NeuMiss}.

In MCAR settings, figure \ref{fig:boxplotmechanism} shows that, as
expected, EM gives the best results when tractable. Yet, we could
not run it for number of features $d \geq 50$. \name is the best
performing method behind EM, in all cases except for $n=2\times 10^{4},
d=50$, where depth of 1 or greater overfit due to the low ratio of
number of parameters to number of samples. In such situation, MLP has
the same expressive power and performs slightly better.
Note that for a high samples-to-parameters ratio
($n=1\times 10^{5}, d=10$), \name reaches an almost perfect $R
2$ score, less than 1\% below the Bayes rate. The results for the MAR setting are very similar to the MCAR results, and are given in supplementary figure \ref{fig:boxplotMAR}.

For the self-masking mechanisms, the \name network significantly
improves upon the competitors, followed by the MLP. This is even true
for the probit self-masking case for which we have no theoretical results.
The gap between the two
architectures widens as the number of samples increases, with the \name
network benefiting from a large amount of data. These results emphasize
the robustness of \name and MLP to the missing data mechanism,
including MNAR settings in which EM or conditional imputation do not
enable statistical analysis.


\section{Discussion and conclusion}

Traditionally, statistical models are adapted to missing values
using EM or imputation. However, these require strong
assumptions on the missing values. Rather, we frame the problem as a risk
minimization with a flexible yet tractable function family.
We propose the \name network, a theoretically-grounded architecture
that handles missing values
using multiplication by the mask as nonlinearities.
It targets the Bayes predictor with
differentiable approximations of the inverses of the various covariance submatrices, thereby reducing complexity by sharing parameters across missing data patterns.
Strong connections between a shallow version of our architecture and
the common practice of inputing the mask to an MLP is established.

The \name architecture has clear practical
benefits. It is robust to the missing-values mechanism, often
unknown in practice. Moreover its sample and computational complexity
are independent of the number of missing-data patterns, which allows to
work with datasets of higher dimensionality and limited sample sizes.
This work opens many perspectives, in particular using
this network as a building block in larger architectures, \emph{eg} to
tackle nonlinear problems. 




\section*{Broader Impact}

In our work, we proposed theoretical foundations to justify the use of a
specific neural network architecture in the presence of missing-values.

Neural networks are known for their challenging black-box nature. We believe that such theory leads to a better understanding of the mechanisms at work in neural networks.

Our architecture is tailored for missing data. These are present in
many applications, in particular in social or health data. In these
fields, it is common for under-represented groups to exhibit a higher
percentage of missing values (MNAR mechanism). Dealing with these missing
values will definitely improve prediction for these groups, thereby
reducing potential bias against these exact same groups.

As any predictive algorithm, our proposal can be misused in a variety of
context, including in medical science, for which a proper assessment of
the specific characteristics of the algorithm output is required
(assessing  bias in prediction, prevent false conclusion resulting from
misinterpreting outputs). Yet, by improving performance and
understanding of a fundamental challenge in many applications settings,
our work is not facilitating more unethical aspects of AI than ethical
applications. Rather, medical studies that suffer chronically from
limited sample sizes are mostly likely to benefit from the reduced sample
complexity that these advances provide.

\begin{ack}
This work was funded by ANR-17-CE23-0018 - DirtyData - Intégration et nettoyage de données pour l'analyse statistique (2017) and the MissingBigData grant from DataIA.

\end{ack}

\small

\bibliographystyle{plainnat}
\bibliography{biblio}

\clearpage
\appendix

\toptitlebar
{\centering{\Large{\bfseries  Supplementary materials} -- \mytitle \par}}
\bottomtitlebar

\section{Proofs}

\subsection{Proof of Lemma~\ref{lem:bayes_predictor}}

\begin{lemma}[General expression of the Bayes predictor]
 \label{lem:bayes_predictor}
 Assume that the data are generated via the linear model defined in equation~\eqref{eq:predictor}, then the Bayes predictor takes the form
 \begin{equation}
     f^{\star} (X_{obs(M)},M) = \beta_0^{\star} + \langle \beta_{obs(M)}^{\star}, X_{obs(M)} \rangle  + \langle \beta^{\star}_{mis(M)},  \E[X_{mis(M)}|M, X_{obs(M)}] \rangle,
     \label{eq:bayespredictgeneral}
 \end{equation}
 where
 ($\beta_{obs(M)}^{\star}, \beta_{mis(M)}^{\star}$) correspond to the decomposition of the regression coefficients in observed and missing elements.
 \end{lemma}

\begin{proof}[Proof of Lemma~\ref{lem:bayes_predictor}]
By definition of the linear model, we have
\begin{align*}
    \regTilde (\bXm) & = \mathbb E[Y | \bXm]\\
    & = \mathbb E[\beta^{\star}_0+ \langle \beta^{\star},  X \rangle~|~M, X_{obs(M)}] \\
    & = \beta_0^{\star} + \langle \beta^{\star}_{obs(M)},  X_{obs(M)} \rangle + \langle \beta^{\star}_{mis(M)}, \mathbb E[X_{mis(M)}~|~M, X_{obs(M)}] \rangle.
\end{align*}
\end{proof}

\subsection{Proof of Lemma~\ref{lem:product}}

\begin{lemma}[Product of two multivariate gaussians] \label{lem:product}
Let $f(X) = \exp\br{(X - a)^\top A^{-1}(X-a)}$ and $g(X) = \exp\br{(X - b)^\top B^{-1}(X-b)}$ be two Gaussian functions, with $A$ and $B$ positive semidefinite matrices. Then the product $f(X)g(X)$ is another gaussian function given by:
$$f(X)g(X) = \exp\br{-\frac{1}{2} (a-b)^\top(A+B)^{-1}(a-b))} \exp\br{-\frac{1}{2} (X - \mu_p)^\top \Sigma_p^{-1} (X - \mu_p)}$$
where $\mu_p$ and $\Sigma_p$ depend on $a$, $A$, $b$ and $B$.
\end{lemma}

\begin{proof}[Proof of Lemma~\ref{lem:product}]
Identifying the second and first order terms in $X$ we get:
\begin{align}
    \Sigma_p^{-1} &= A^{-1} + B^{-1} \label{eq:mu_p}\\
    \Sigma_p^{-1} \mu_p &= A^{-1}a + B^{-1}b \label{eq:Sigma_p}
\end{align}
By completing the square, the product can be rewritten as:
\begin{equation*}
    f(X)g(X) = \exp \br{-\frac{1}{2}(a^\top A^{-1}a + b^\top B^{-1}b - \mu_p^\top \Sigma_p^{-1} \mu_p} \exp\br{-\frac{1}{2} (X - \mu_p)^\top \Sigma_p^{-1} (X - \mu_p)}
\end{equation*}
Let's now simplify the scaling factor:
\begin{align*}
    c &= a^\top A^{-1}a + b^\top B^{-1}b - \mu_p^\top \Sigma_p^{-1} \mu_p\\
    &= a^\top A^{-1}a + b^\top B^{-1}b - \br{a^\top A^{-1}(A^{-1}+B^{-1})^{-1} + b^\top B^{-1}(A^{-1} + B^{-1})^{-1}}\br{A^{-1}a + B^{-1}b}\\
    &= a^\top(A^{-1} - A^{-1}(A^{-1} + B^{-1})^{-1}A^{-1})a + b^\top(B^{-1} - B^{-1}(A^{-1} + B^{-1})^{-1}B^{-1})b \\
    & \quad - 2a^\top (A^{-1}(A^{-1}+B^{-1})^{-1}B^{-1})b\\
    &= a^\top(A+B)^{-1}a + b^\top(A+B)^{-1}b - 2a^\top(A+B)^{-1}b\\
    &= (a-b)^\top(A+B)^{-1}(a-b)
\end{align*}
The third equality is true because $A$ and $B$ are symmetric. The fourth equality uses the Woodbury identity and the fact that:
\begin{align*}
    (A^{-1}(A^{-1}+B^{-1})^{-1}B^{-1}) &= \br{B(A^{-1} + B^{-1})A}^{-1}\\
    &= \br{BA^{-1}A + BB^{-1}A}^{-1}\\
    &= \br{B + A}^{-1}
\end{align*}
The last equality allows to conclude the proof.
\end{proof}

\subsection{Proof of Proposition~\ref{prop:MCAR_MAR}}
\label{app:sec:proof_prop1}

\MARbp*

Lemma~\ref{lem:bayes_predictor} gives the general expression of the Bayes predictor for any data distribution and missing data mechanism. From this expression, on can see that the crucial step to compute the Bayes predictor is computing $\E[X_{mis} | M, X_{obs}]$, or in other words, $\E[X_{j} | M, X_{obs}]$ for all $j \in mis$. In order to compute this expectation, we will characterize the distribution $P(X_j|M, X_{obs})$ for all $j \in mis$. Let $mis^\prime(M, j) = mis(M)\setminus\{j\}$. For clarity, when there is no ambiguity we will just write $mis^\prime$. Using the sum and product rules of probability, we have:
\begin{align}
    P(X_j|M, X_{obs}) &= \frac{P(M, X_j, X_{obs})}{P(M, X_{obs})}\\
    &= \frac{\int P(M, X_j, X_{obs}, X_{mis^\prime}) \mathrm{d}X_{mis^\prime} }{\int \int P(M, X_j, X_{obs}, X_{mis^\prime}) \mathrm{d}X_{mis^\prime} \mathrm{d}X_j}\\
    &= \frac{\int P(M | X_{obs}, X_j, X_{mis^\prime}) P(X_{obs}, X_j, X_{mis^\prime}) \mathrm{d}X_{mis^\prime}}{\int \int P(M | X_{obs}, X_j, X_{mis^\prime}) P(X_{obs}, X_j, X_{mis^\prime})  \mathrm{d}X_{mis^\prime} \mathrm{d}X_j} \label{eq:general_factorisation}
\end{align}
In the MCAR case, for all $\bbm \in \{0,1\}^d, \P (\bM = \bbm | \bX) = \P (\bM = \bbm)$, thus we have
\begin{align}
    P(X_j|M, X_{obs}) &= \frac{P(M) \int P(X_{obs}, X_j, X_{mis^\prime}) \mathrm{d}X_{mis^\prime}}{P(M) \int \int P(X_{obs}, X_j, X_{mis^\prime}) \mathrm{d}X_{mis^\prime} \mathrm{d}X_j}\\
    &= \frac{P(X_{obs}, X_j)}{P(X_{obs})}\\
    &= P(X_j|X_{obs}) \label{eq:mcar_proba}
\end{align}
On the other hand, assuming MAR mechanism, that is, for all $\bbm \in \{0,1\}^d$, $P(\bM = \bbm|\bX) = P(\bM = \bbm|\bX_{obs(\bbm)})$, we have, given equation~\eqref{eq:general_factorisation},
\begin{align}
    P(X_j|M, X_{obs})  &= \frac{P(M |X_{obs}) \int P(X_{obs}, X_j, X_{mis^\prime}) \mathrm{d}X_{mis^\prime}}{P(M | X_{obs}) \int \int P(X_{obs}, X_j, X_{mis^\prime}) \mathrm{d}X_{mis^\prime} \mathrm{d}X_j}\\
    &= \frac{P(X_{obs}, X_j)}{P(X_{obs})}\\
    &= P(X_j|X_{obs}) \label{eq:mar_proba}
\end{align}
Therefore, if the missing data mechanism is MCAR or MAR, we have, according to equation~\eqref{eq:mcar_proba} and \eqref{eq:mar_proba},
\begin{align*}
\mathbb E[X_{mis(M)}~|~M, X_{obs(M)}] = \mathbb E[X_{mis(M)}~| X_{obs(M)}].
\end{align*}
Since $X$ is a Gaussian vector distributed as $\mathcal{N}(\mu, \Sigma)$, we know that the conditional expectation $\E[X_{mis(M)}~| X_{obs(M)}]$ satisfies
\begin{align}
    \mathbb E\left[X_{mis(m)}~\middle|~X_{obs(m)}\right]  = \mu_{mis(m)} + \Sigma_{mis(m),obs(m)} \left(\Sigma_{obs(m)}\right)^{-1}\left(X_{obs(m)}-\mu_{obs(m)}\right), \label{eq:mis_cond_obs}
\end{align}
\citep[see, e.g.,][]{majumdar2019conditional}. This concludes the proof according to Lemma~\ref{lem:bayes_predictor}.

\subsection{Proof of Proposition~\ref{prop:sm}}
\label{app:proof_prop_mnar}

\GaussianSMbp*

In the Gaussian self-masking case, according to Assumption~\ref{ass:selfmasked_gaussian}, the probability factorizes as $P(M =m| X) = P(M_{mis(m)}=1|X_{mis(m)}) P(M_{obs(m)}=0|X_{obs(m)})$. Equation~\ref{eq:general_factorisation} can thus be rewritten as:
\begin{align}
    P(X_j|M, X_{obs}) &=  \frac{P(M_{obs}=0 |X_{obs}) \int P(M_{mis}=1 | X_{mis}) P(X_{obs}, X_j, X_{mis^\prime}) \mathrm{d}X_{mis^\prime}}{P(M_{obs}=0 | X_{obs}) \int \int P(M_{mis}=1 | X_{mis}) P(X_{obs}, X_j, X_{mis^\prime}) \mathrm{d}X_{mis^\prime} \mathrm{d}X_j}\\
    &= \frac{\int P(M_{mis}=1 | X_{mis}) P(X_{mis}|X_{obs}) \mathrm{d}X_{mis^\prime}}{\int \int P(M_{mis}=1 | X_{mis}) P(X_{mis}|X_{obs}) \mathrm{d}X_{mis^\prime} \mathrm{d}X_j} \label{eq:sm_factorisation}
\end{align}

Let $D$ be the diagonal matrix such that $\mathrm{diag}(D) = \tsigma^2$, where $\tsigma$ is defined in Assumption~\ref{ass:selfmasked_gaussian}. Then the masking probability reads:
\begin{equation}
    \label{eq:masking_proba}
    P(M_{mis}=1 |X_{mis}) = \prod_{k \in mis}^d K_k \exp\br{-\frac{1}{2} (X_{mis} - \tmu_{mis})(D_{mis, mis})^{-1}(X_{mis} - \tmu_{mis})}
\end{equation}

Using the conditional Gaussian formula, we have $P(X_{mis} |X_{obs}) = \mathcal{N} (X_{mis}| \mu_{mis|obs}, \Sigma_{mis|obs})$ with
\begin{align}
    \mu_{mis|obs} &= \mu_{mis} + \Sigma_{mis, obs} \Sigma_{obs,  obs}^{-1} \br{X_{obs} - \mu_{obs}} \label{eq:conditional_mean} \\
    \Sigma_{mis|obs} &= \Sigma_{mis, mis} - \Sigma_{mis, obs} \Sigma_{obs}^{-1} \Sigma_{obs, mis} \label{eq:conditional_cov}
\end{align}
Thus, according to equation~\eqref{eq:masking_proba}, $P(M_{mis}=1 | X_{mis})$ and $P(X_{mis}|X_{obs})$ are Gaussian functions of $X_{mis}$. By Lemma~\ref{lem:product}, their product is also a Gaussian function given by:
\begin{equation}
    P(M_{mis}=1 | X_{mis})P(X_{mis}|X_{obs}) = K \exp\br{-\frac{1}{2} (X_{mis} - a_M)^\top \br{A_M}^{-1} (X_{mis} - a_M)}
\end{equation}
where $a_M$ and $A_M$ depend on the missingness pattern and
\begin{gather}
    K = \prod_{k \in mis}^d \frac{K_k}{\sqrt{(2\pi)^{|mis|} |\Sigma_{mis|obs}|}} \exp\br{-\frac{1}{2} (\tmu_{mis}-\mu_{mis|obs})^\top(\Sigma_{mis|obs}+D_{mis, mis})^{-1}(\tmu_{mis}-\mu_{mis|obs})}\\
    \br{A_M}^{-1} = D_{mis, mis}^{-1} + \Sigma_{mis|obs}^{-1} \label{eq:A_M}\\
    \label{eq:a_M}
    \br{A_M}^{-1} a_M = D_{mis, mis}^{-1}\tmu_{mis} + \Sigma_{mis|obs}^{-1}\mu_{mis|obs}
\end{gather}

Because $K$ does not depend on $X_{mis}$, it simplifies from eq~\ref{eq:sm_factorisation}. As a result we get:
\begin{align}
    P(X_j|M, X_{obs}) &=  \frac{\int \Ncal(X_{mis}|a_M, A_M) \mathrm{d}X_{mis^\prime}}{\int \int \Ncal(X_{mis}|a_M, A_M) \mathrm{d}X_{mis^\prime} \mathrm{d}X_j}\\
    &= \mathcal{N} (X_j|(a_M)_j, (A_M)_{j,j})
\end{align}

By definition of the Bayes predictor, we have
\begin{align}
    \regTilde (\bXm) = \beta_0^{\star} + \langle \beta_{obs(M)}^{\star}, X_{obs(M)} \rangle  + \langle \beta^{\star}_{mis(M)},  \E[X_{mis(M)}|M, X_{obs(M)}] \rangle,
\end{align}
where
\begin{equation}
    \label{eq:sm_bayes_predictor}
    \E[X_{mis} | M, X_{obs}] = (a_M)_{mis}.
\end{equation}
Combining equations~\eqref{eq:A_M}, \eqref{eq:a_M}, \eqref{eq:sm_bayes_predictor}, we obtain
\begin{align}
\E[X_{mis}|M, X_{obs}] =& \br{Id + D_{mis}\Sigma_{mis| obs}^{-1}}^{-1} \\
& \times \sqb{\tilde{\mu}_{mis} + D_{mis}\Sigma_{mis| obs}^{-1}\br{\mu_{mis} + \Sigma_{mis, obs} \br{\Sigma_{obs}}^{-1}\br{X_{obs}-\mu_{obs}}}}
\label{eq:bp_gaussian_sm}
\end{align}

\subsection{Controlling the convergence of Neumann iterates}

Here we establish an auxiliary result, controlling the convergence of
Neumann iterates to the matrix inverse.
\label{sec:conv_neumann_iterates}

\begin{proposition}[Linear convergence of Neumann iterations]
\label{prop:conv_neumann_iterates}
    Assume that the spectral radius of $\Sigma$ is strictly less than $1$. Therefore, for all missing data patterns $m \in \{0,1\}^d$, the iterates $S^{(\ell)}_{obs(m)}$ defined in equation~\eqref{eq:iterative_algo} converge linearly towards $(\Sigma_{obs(m)})^{-1}$ and satisfy, for all $\ell \geq 1$,
    \[
        \| Id -  \Sigma_{obs(m)}S_{obs(m)}^{(\ell)}\|_2
                \le (1 - \nu_{obs(m)})^{\ell}
                    \| Id -  \Sigma_{obs(m)}S^{(0)}_{obs(m)}\|_2
        \enspace ,
    \]
    where $\nu_{obs(m)}$ is the smallest eigenvalue of $\Sigma_{obs(m)}$.
\end{proposition}
Note that Proposition~\ref{prop:conv_neumann_iterates} can easily be
extended to the general case by working with $\Sigma / \rho(\Sigma)$ and
multiplying the resulting approximation by $\rho(\Sigma)$, where
$\rho(\Sigma)$ is the spectral radius of $\Sigma$.

\begin{proof}

Since the spectral radius of $\Sigma$ is strictly smaller than one, the spectral radius of each submatrix $\Sigma_{obs(m)}$ is also strictly smaller than one. This is a direct application of Cauchy Interlace Theorem \citep{Hwang2004} or it can be seen with the definition of the  eigenvalues
\[
   \rho(\Sigma_{obs(m)}) = \max_{u\in\RR^{|obs(m)|}} u^\top \Sigma_{obs(m)}u = \max_{\substack{x\in\RR^d\\x_{mis} = 0}} x^\top \Sigma x \le \max_{x \in \RR^{d}} x^\top\Sigma x = \rho(\Sigma)
   \enspace .
\]
Note that $S_{obs(m)}^{\ell} = \sum_{k=0}^{\ell-1} \br{Id - \Sigma_{obs}}^k +  \br{Id - \Sigma_{obs}}^{\ell}S^0_{obs(m)}$
can be defined recursively via the iterative formula
\begin{align}
    S_{obs(m)}^{\ell} &= (Id -  \Sigma_{obs(m)}) S_{obs(m)}^{\ell-1}  + Id \label{eq:neumann_iteration}
\end{align}

The matrix $(\Sigma_{obs(m)})^{-1}$ is a fixed point of the Neumann iterations (equation~\eqref{eq:neumann_iteration}). It verifies the following equation
    \begin{equation}
         (\Sigma_{obs(m)})^{-1} = (Id -  \Sigma_{obs(m)})  (\Sigma_{obs(m)})^{-1} + Id
        \enspace . \label{eq:neumann_iteration_sigma}
    \end{equation}
    By substracting \ref{eq:neumann_iteration} to this equation, we obtain
    \begin{equation}
        (\Sigma_{obs(m)})^{-1} - S_{obs(m)}^{\ell} = (Id -  \Sigma_{obs(m)})( (\Sigma_{obs(m)})^{-1} - S_{obs(m)}^{\ell-1})
        \enspace .
    \end{equation}
    Multiplying both sides by $\Sigma_{obs(m)}$ yields
    \begin{equation}
        (Id - \Sigma_{obs(m)}S_{obs(m)}^{\ell}) = (Id -  \Sigma_{obs(m)})( Id - \Sigma_{obs(m)} S_{obs(m)}^{\ell-1})
        \enspace .
    \end{equation}
    Taking the $\ell_2$-norm and using Cauchy-Schwartz inequality yields
    \begin{equation}
        \|Id - \Sigma_{obs(m)}S_{obs(m)}^{\ell}\|_2
            \le \|Id -  \Sigma_{obs(m)}\|_2
                \|Id - \Sigma_{obs(m)}S_{obs(m)}^{\ell-1}\|_2
        \enspace .
    \end{equation}
    Let $\nu_{obs(m)}$ be the smallest eigenvalue of $\Sigma_{obs(m)}$, which is positive since $\Sigma$ is invertible. Since the largest eigenvalue of $\Sigma_{obs(m)}$ is upper bounded by $1$, we get that $\|Id - \widetilde  \Sigma\|_2 = (1 - \nu_{obs(m)})$ and by recursion we obtain
    \begin{equation}
        \| Id -  \Sigma_{obs(m)}S_{obs(m)}^{\ell}\|_2
            \le (1 - \nu_{obs(m)})^\ell
                \| Id -  \Sigma_{obs(m)}S^0_{obs(m)}\|_2 \enspace .
    \end{equation}
\end{proof}

\subsection{Proof of Proposition~\ref{prop:bayes_bound_approx}}

    \approxBayes*

    According to Proposition~\ref{prop:MCAR_MAR} and the definition of the approximation of order $p$ of the Bayes predictor (see equations~\eqref{eq:bp_MCAR_any_order})
    \[
        \regTildel (\bXm) =  \langle \beta^\star_{obs},  \bX_{obs} \rangle + \langle \beta^\star_{mis}, \mu_{mis} +  \Sigma_{mis, obs} S^{(\ell)}_{obs} \br{X_{obs}-\mu_{obs}} \rangle
        \enspace ,
    \]
    Then
    \begin{align}
        & \E[(\regTildel(\bXm) - \regTilde(\bXm))^2]\\
                & =  \E\Big[ \big\langle \beta_{mis}^{\star} ~,~
                    \Sigma_{mis, obs}(S^{\ell}_{obs} - \Sigma_{obs}^{-1})
                    (X_{obs} - \mu_{obs}) \big\rangle ^2\Big]\\
                & =  \E\Big[(\beta_{mis}^{\star})^\top
                    \Sigma_{mis, obs}(S^{\ell}_{obs} - \Sigma_{obs}^{-1})
                    (X_{obs} - \mu_{obs})(X_{obs} - \mu_{obs})^\top
                    (S^{\ell}_{obs} - \Sigma_{obs}^{-1})\Sigma_{obs,mis}
                    \beta_{mis}^{\star}\Big]\\
                & =  \E\Big[(\beta_{mis}^{\star})^\top
                    \Sigma_{mis, obs}(S^{\ell}_{obs} - \Sigma_{obs}^{-1})
                    \underbrace{\E[(X_{obs} - \mu_{obs})(X_{obs} - \mu_{obs})^\top|M]}_{\Sigma_{obs}}
                    (S^{\ell}_{obs} - \Sigma_{obs}^{-1})\Sigma_{obs, mis}
                    \beta_{mis}^{\star}\Big]\\
                & =  \E\Big[(\beta_{mis}^{\star})^\top
                    \Sigma_{mis, obs}
                    (S^{\ell}_{obs} - \Sigma_{obs}^{-1})
                    \Sigma_{obs}
                    (S^{\ell}_{obs} - \Sigma_{obs}^{-1})
                    \Sigma_{obs,mis}
                    \beta_{mis}^{\star}\Big]\\
                & =  \E\Big[\big\|(\Sigma_{obs})^{\frac12} (\Sigma_{obs})^{-1}
                    (\Sigma_{obs}S^{\ell}_{obs} - Id_{obs})
                    \Sigma_{obs, mis}
                    \beta_{mis}^{\star}\big\|_2^2\Big]\\
                & =  \E\Big[\big\|(\Sigma_{obs})^{-\frac12}
                    (Id_{obs} - \Sigma_{obs}S^{\ell}_{obs})
                    \Sigma_{obs, mis}
                    \beta_{mis}^{\star}\big\|_2^2\Big]\label{eq:cvg:sigma_obs}\\
                & \le \|\Sigma^{-1}\|_2 \|\Sigma\|_2^2\|\beta^{\star}\|_2^2
                    \E\big[\|Id_{obs} - \Sigma_{obs}S^{\ell}_{obs} \|_2^2\big]\label{eq:cvg:sigma}\\
                & \le \frac{1}{\nu} \|\beta^{\star}\|_2^2
                    \E\big[(1 - \nu_{obs})^{2\ell}\|Id_{obs} - \Sigma_{obs}S^{0}_{obs} \|_2^2\big]
    \end{align}
    An important point for going from \eqref{eq:cvg:sigma_obs} to \eqref{eq:cvg:sigma} is to notice that for any missing pattern, we have
    \begin{align*}
        \|\Sigma_{obs,mis}\|_2 \le \|\Sigma\|_2 \text{ and }  \|\Sigma_{obs}^{-1}\|_2 \le \|\Sigma^{-1}\|_2
        \enspace .
    \end{align*}
    The first inequality can be obtained by observing that computing the largest singular value of $\Sigma_{obs,mis}$ reduces to solving a constrained version of the maximization problem that defines the largest eigenvalue of $\Sigma$:
    \begin{align*}
        \|\Sigma_{obs,mis}\|_2
            = \max_{\|x_{mis}\|_2 = 1} \|\Sigma_{obs,mis} x_{mis}\|_2
            \le \max_{\substack{\|x\|_2 = 1\\x_{obs}=0}} \|\Sigma_{obs, \cdot} x\|_2
            \le \max_{\substack{\|x\|_2 = 1\\x_{obs}=0}} \|\Sigma x\|_2
            \le \max_{\|x\|_2 = 1} \|\Sigma x\|_2^2 =\|\Sigma\|_2
            \enspace .
    \end{align*}
    where we used $\|\Sigma_{obs, \cdot} x\|_2^2 = \sum_{i \in obs} (\Sigma_i^\top x)^2 \le \sum_{i=1}^d (\Sigma_i^\top x)^2 = \|\Sigma x\|_2^2$.\\
    A similar observation can be done for computing the smallest eigenvalue
    of $\Sigma$, $\lambda_{\min}(\Sigma)$:
    \[
        \lambda_{\min}(\Sigma)
            = \min_{\|x\|_2 = 1} x^\top\Sigma x
            \le \min_{\substack{\|x\|_2 = 1\\x_{mis}=0}} x^\top\Sigma x
            = \min_{\|x_{obs}\|_2 = 1} x^\top_{obs}\Sigma_{obs} x_{obs}
            = \lambda_{\min}(\Sigma_{obs})
            \enspace .
    \]
and we can deduce the second inequality by noting that $\lambda_{\min}(\Sigma) = \frac{1}{\| \Sigma^{-1}\|_2^2}$ and $\lambda_{\min}(\Sigma_{obs}) = \frac{1}{\| \Sigma_{obs}^{-1}\|_2^2}$.

\subsection{Proof of Proposition~\ref{prop:link_MLP_Neumann}}

\eqNeumannMLP*

\paragraph{Obtaining a $\odot M$ nonlinearity from a ReLU nonlinearity.}
Let $\mathcal{H}_{ReLU} = \br{\sqb{W^{(X)}, W^{(M)}} \in \RR^{d \times 2d}, ReLU}$ be a hidden layer which connects $\sqb{X, M}$ to $d$ hidden units, and applies a ReLU nonlinearity to the activations. We denote by $b \in \RR^d$ the bias corresponding to this layer. Let $k \in \bbr{1, d}$. Depending on the missing data pattern that is given as input, the $k^{th}$ entry can correspond to either a missing or an observed entry. We now write the activation of the $k^{th}$ hidden unit depending on whether entry $k$ is observed or missing. The activation of the $k^{th}$ hidden unit is given by
\begin{align}
    a_k &= W_{k, .}^{(X)} X + W_{k, .}^{(M)} M + b_k\\
    &=  W_{k, obs}^{(X)} X_{obs} + W_{k, mis}^{(M)} \mathbf{1}_{mis} + b_k. \label{eq:proof_1}
\end{align}
Emphasizing the role of $W_{k,k}^{(M)}$ and $W_{k,k}^{(X)}$, we can decompose equation~\eqref{eq:proof_1} depending on whether the $k^{th}$ entry is observed or missing
\begin{align}
    \label{eq:activation_R}
    \text{If } k \in mis, \quad a_k &= W_{k, obs}^{(X)} X_{obs} + W_{k, k}^{(M)} + W^{(M)}_{k, mis\setminus\{k\}} \mathbf{1}_{k, mis\setminus\{k\}} + b_k\\
    \text{If } k \in obs, \quad a_k &= W_{k, k}^{(X)} X_k + W_{k, obs\setminus\{k\}}^{(X)} X_{obs\setminus\{k\}} + W^{(M)}_{k, mis} \mathbf{1}_{mis} + b_k.
\end{align}

Suppose that the weights $W^{(X)}$ as well as $W^{(M)}_{i, j}, i\ne j$ are fixed. Then, under the assumption that the support of $X$ is finite, there exists a bias $b^*_k$ which verifies:
\begin{equation}
\label{eq:negative_activation}
    \forall X, \quad a_k = W_{k, k}^{(X)} X_k + W_{k, obs\setminus\{k\}}^{(X)} X_{obs\setminus\{k\}} + W^{(M)}_{k, mis} \mathbf{1}_{mis} + b^*_k \leq 0
\end{equation}
i.e., there exists a bias $b^*_k$ such that the activation of the $k^{th}$ hidden unit is always negative when $k$ is observed. Similarly, there exists $W_{k, k}^{*, (M)}$ such that:
\begin{equation}
\label{eq:positive_activation}
    \forall X, \quad a_k = W_{k, obs}^{(X)} X_{obs} + W_{k, k}^{*, (M)} + W^{(M)}_{k, mis\setminus\{k\}} \mathbf{1}_{k, mis\setminus\{k\}} + b^*_k \geq 0
\end{equation}
i.e., there exists a weight $W_{k, k}^{*, (M)}$ such that the activation of the $k^{th}$ hidden unit is always positive when $k$ is missing. Note that these results hold because the weight $W_{k, k}^{(M)}$ only appears in the expression of $a_k$ when entry $k$ is missing. Let $h_k = ReLU(a_k)$. By choosing $b_k=b^*_k$ and $W_{k, k}^{(M)} = W_{k, k}^{*, (M)}$, we have that:
\begin{align}
    \text{If } k \in mis, \quad h_k &= a_k\\
    \text{If } k \in obs, \quad h_k &= 0
\end{align}
As a result, the output of the hidden layer $\mathcal{H}_{ReLU}$ can be rewritten as:
\begin{equation}
    h_k = a_k \odot M
\end{equation}
i.e., a $\odot M$ nonlinearity is applied to the activations.

\paragraph{Equating the slopes and biases of $\mathcal{H}_{ReLU}$ and $\mathcal{H}_{\odot M}$.}
Let $\mathcal{H}_{\odot M} = \br{W \in \RR^{d \times d}, \mu, \odot M}$ be the layer that connect $(X-\mu)\odot(1-M)$ to $d$ hidden units via the weight matrix $W$, and applies a $\odot M$ nonlinearity to the activations. We will denote by $c \in \RR^d$ the bias corresponding to this layer.

The activations for this layer are given by:
\begin{align}
    a_k &= W_{k, obs}(X_{obs} - \mu_{obs}) + c_k\\
    \label{eq:activation_N}
    &= W_{k, obs}X_{obs} - W_{k, obs}\mu_{obs} + c_k
\end{align}

Then by applying the non-linearity we obtain the output of the hidden layer:
\begin{align}
    \text{If } k \in mis, \quad h_k &= a_k\\
    \text{If } k \in obs, \quad h_k &= 0
\end{align}

It is straigthforward to see that with the choice of $b_k=b^*_k$ and $W_{k, k}^{(M)} = W_{k, k}^{*, (M)}$ for $\mathcal{H}_{ReLU}$, both hidden layers have the same output $h_k = 0$ when entry $k$ is observed. It remains to be shown that there exists a configuration of the weights of $\mathcal{H}_{ReLU}$ such that the activations $a_k$ when entry $k$ is missing are equal to those of $\mathcal{H}_{\odot M}$. To avoid confusions, we will now denote by $a
^{(N)}_k$ the activations of $\mathcal{H}_{\odot M}$ and by $a^{(R)}_k$ the activations of $\mathcal{H}_{ReLU}$. We recall here the activations for both layers as derived in \ref{eq:activation_N} and \ref{eq:activation_R}.
\begin{equation}
\text{If } k \in mis,
    \begin{cases}
        a_k^{(N)} = W_{k, obs}X_{obs} - W_{k, obs}\mu_{obs} + c_k\\
        a_k^{(R)} = W_{k, obs}^{(X)} X_{obs} + W_{k, k}^{*, (M)} + W^{(M)}_{k, mis\setminus\{k\}} \mathbf{1}_{k, mis\setminus\{k\}} + b_k^*
    \end{cases}
\end{equation}

By setting $W_{k, .}^{(X)} = W_{k, .}$, we obtain that both activations have the same slopes with regards to $X$. We now turn to the biases. We have that:
\begin{align}
    W_{k, k}^{*, (M)} + W^{(M)}_{k, mis\setminus\{k\}} \mathbf{1}_{k, mis\setminus\{k\}} + b_k^* = W_{k, .}^{(M)} \mathbf{1} - W^{(M)}_{k, obs}\mathbf{1} + b_k^*
\end{align}
We now set:
\begin{align}
    \label{eq:biais1}
    \forall j \in obs,\quad &W^{(M)}_{kj} = W_{kj}\mu_j\\
    \label{eq:biais2}
    &W_{k.}^{(M)}\mathbf{1} + b_k^* = c_k
\end{align}
to obtain that both activations have the same biases. Note that \ref{eq:biais1} sets the weights $W_{k, j}$ for all $j \ne k$ (since $obs$ can contain any entries except $k$). As a consequence, equation \ref{eq:biais2} implies an equation invloving $W_{kk}^{*, (M)}$ and $b_k^*$ where all other parameters have already been set. Since $W_{kk}^{*, (M)}$ and $b_k^*$ are also chosen to satisfy the inequalities \ref{eq:negative_activation} and \ref{eq:positive_activation}, it may not be possible to choose them so as to also satify equation~\ref{eq:biais2}. As a result, the functions computed by the activated hidden units of $\mathcal{H}_{ReLU}$ can be equal to those computed by $\mathcal{H}_{\odot M}$ up to a constant.

\section{Additional results}

\subsection{\name network scaling law in MNAR}

\begin{figure}[h!]\center
    \includegraphics[height=.35\linewidth]{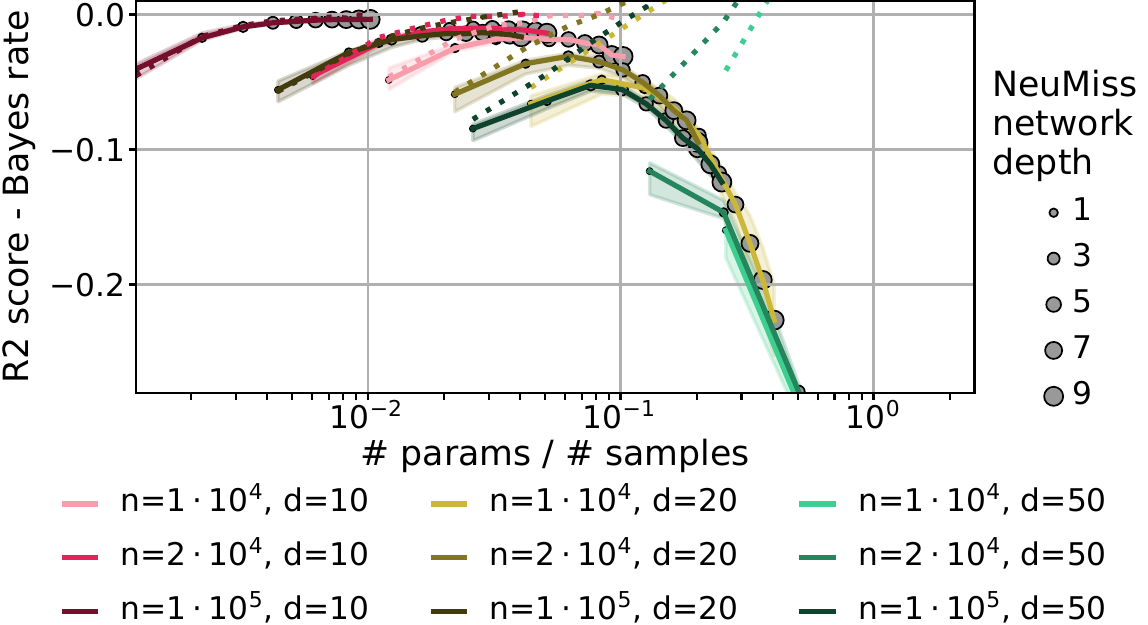}%
    \llap{\raisebox{.15\linewidth}{\sffamily Gaussian self-masking}\qquad\qquad\qquad\qquad}%
    \medskip

    \includegraphics[height=.35\linewidth]{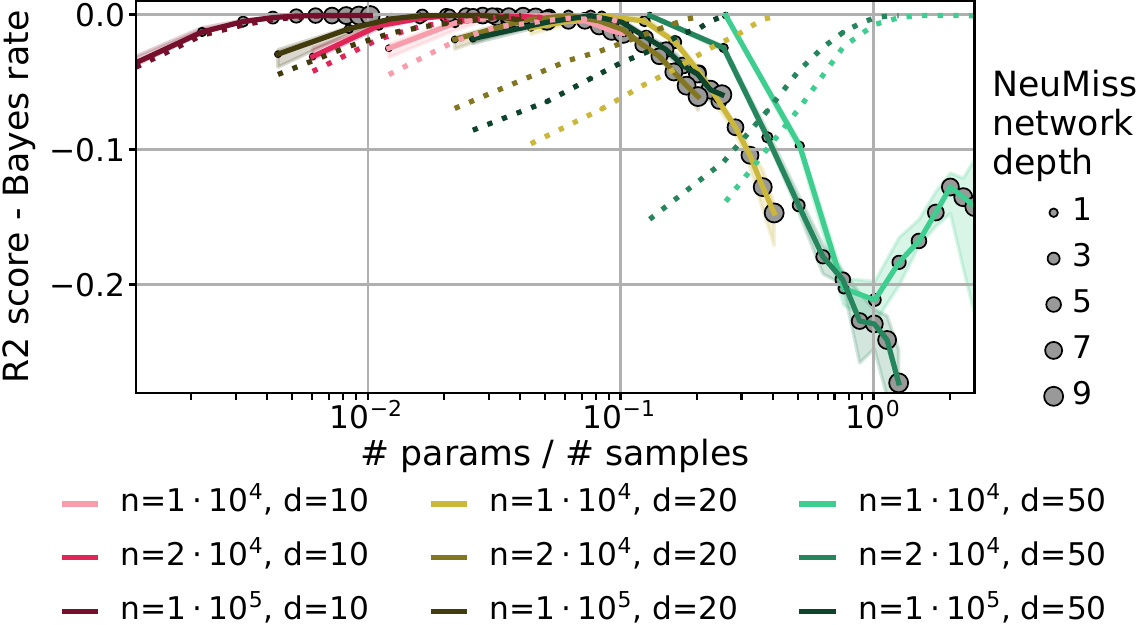}%
    \llap{\raisebox{.15\linewidth}{\sffamily Probit self-masking}\qquad\qquad\qquad\qquad}%

\caption{
\textbf{Required capacity in various MNAR settings} ---
Top: Gaussian self-masking, bottom: probit self-masking.
Performance of \name networks varying the depth in simulations with
different number of samples $n$ and of features $d$.
\label{fig:different_neumann_mnar}}
\end{figure}

\subsection{\name network performances in MAR}

The MAR data was generated as follows: first, a subset of variables with \emph{no} missing values is randomly selected (10\%). The remaining variables have missing values according to a logistic model with random weights, but whose intercept is chosen so as to attain the desired proportion of missing values on those variables (50\%). As can be seen from figure~\ref{fig:boxplotMAR}, the trends observed for MAR are the same as those for MCAR.

\begin{figure}[h!]
    \centering
    MAR\\
    \includegraphics[scale=0.6]{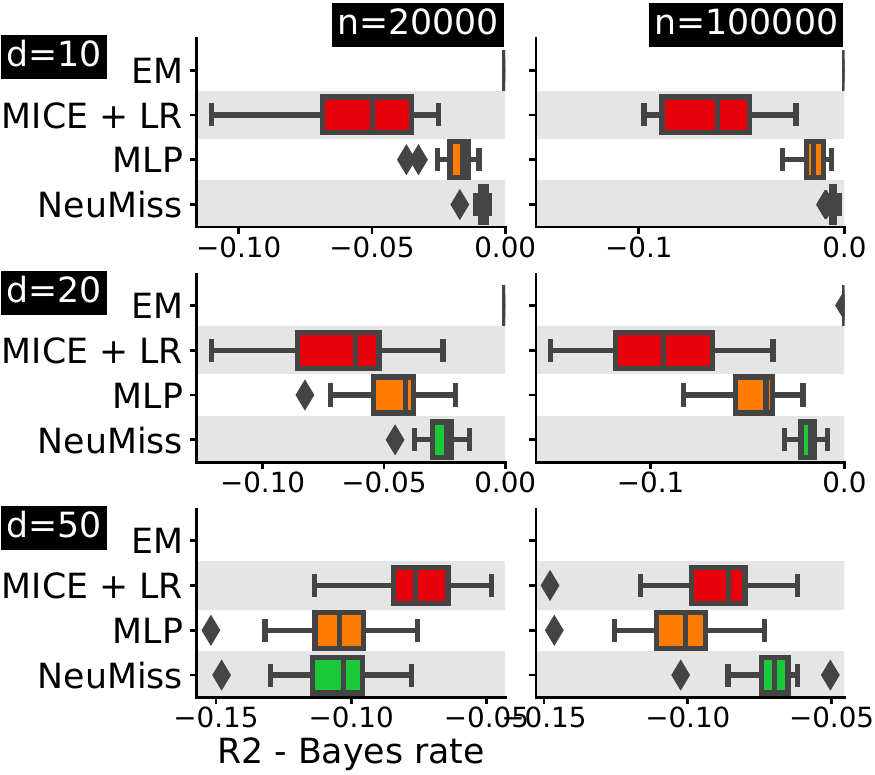}
    \caption{\textbf{Predictive performances in MAR scenario} ---
varying number of samples $n$, and number of features $d$. All experiments are repeated 20 times.}
    \label{fig:boxplotMAR}
\end{figure}

\end{document}